\documentclass{svproc}
\usepackage{graphicx}%
\usepackage{multirow}%
\usepackage{amsmath,amssymb,amsfonts}%
\usepackage{mathrsfs}%
\usepackage[title]{appendix}%
\usepackage{xcolor}%
\usepackage{textcomp}%
\usepackage{manyfoot}%
\usepackage{booktabs}%
\usepackage{algorithm}
\usepackage{algorithmic}
\usepackage{listings}%
\usepackage{url}

\usepackage{subfig}

\begin{document}
\mainmatter              % start of a contribution
\title{A Consistent Diffusion-Based Algorithm for Semi-Supervised Graph Learning}
\titlerunning{Semi-Supervised Graph Learning}  % abbreviated title (for running head)
%                                     also used for the TOC unless
%                                     \toctitle is used
%
\author{Thomas Bonald\thanks{Contact author: \email{thomas.bonald@telecom-paris.fr}} \and Nathan De Lara}
\authorrunning{Thomas Bonald, Nathan De Lara} % abbreviated author list (for running head)

\institute{Institut Polytechnique de Paris, France}
\maketitle              % typeset the title of the contribution

\begin{abstract}
The task of semi-supervised  classification aims at assigning labels to all nodes of a graph based on the  labels known for a few nodes, called the seeds. One of the most popular algorithms relies on the principle of heat diffusion, where the labels of the seeds are spread by thermo-conductance and  the temperature of each node at equilibrium is used as a score function for each label. In this paper, we prove that this algorithm is not consistent unless the temperatures of the nodes at equilibrium are centered before scoring. This crucial step does not only make the algorithm provably consistent on a  block model but  brings significant performance gains on real graphs.
\end{abstract}

\section{Introduction}

%Heat diffusion, describing  the evolution of temperature $T$ in an isotropic material, is governed by the heat equation:
%\begin{equation}\label{eq:heat}
%\dfrac{\partial T}{\partial t} = \alpha \Delta T,
%\end{equation}
%where $\alpha$  is the thermal conductivity of the material and $\Delta$ is the Laplace operator. 
%In  steady state, this equation simplifies to $\Delta T = 0$ and the function $T$ is said to be {\it harmonic}. The Dirichlet problem consists in finding the equilibrium in the presence of  boundary conditions, that is when the temperature $T$ is fixed on the boundary of the region of interest.

The principle of heat diffusion has proved instrumental in graph mining \cite{kondor2002diffusion}. 
  It has been applied for  many different tasks, including   pattern matching  \cite{thanou2017learning},  ranking   \cite{ma2011mining},  embedding   \cite{donnat2018learning},    clustering \cite{tremblay2014graph},  classification \cite{zhu2003semi,zhu2005semi,berberidis2018adadif,DBLP:journals/corr/abs-1902-06105} and feature propagation \cite{pmlr-v198-rossi22a}.
In this paper, we focus on the  task of semi-supervised node classification: given  labels known for a few nodes of the graph, referred to as the {\it seeds}, how to infer the labels of the other nodes? % The number of seeds is typically small compared to the total number of nodes (e.g., 1\%), hence the name of  {\it semi-supervised} learning.
  A popular approach consists in using   diffusion in the graph, under boundary constraints,  a problem known in physics as the Dirichlet problem \cite{zhu2003semi}.
 Specifically, one Dirichlet problem is solved per label, setting at 1 the  temperature of the seeds with this label and at 0 the temperature of the other seeds. Each node is then assigned the label with the highest temperature over the different Dirichlet problems.
In this paper, we prove using a simple block model that this algorithm is actually not consistent, unless the temperatures are {\it centered} before label assignement. This  step of temperature centering does not only make the algorithm consistent but also brings substantial performance gains    on real datasets. 
This is a crucial observation given the popularity of the algorithm\footnote{The number of citations of the paper \cite{zhu2003semi} exceeds 4\,000   in 2023, according to Google Scholar.}.%We propose to solve one Dirichlet problem per label by setting the temperatures of the seeds accordingly, in a one-against-all strategy, and to classify the nodes with respect to the {\it deviation of   temperature to the mean}.  We prove the consistency of our algorithm on a simple block model and its efficiency through experiments  on  real graphs.

%In most cases, the learning task relies on the transient state of the diffusion process, using hot sources only. Our approach is different in that we solve one Dirichlet problem per label, using a one-against-all strategy, with both hot sources (the seeds of the considered label) and cold sources (the seeds of the other labels). The algorithm is parameter-free, unlike  existing techniques based on the heat kernel, whose performance critically depends on some  time parameter used to stop the diffusion. Our theoretical analysis  also shows that  temperature centering is critical, i.e., classification must  rely on  temperature deviations to the mean for each Dirichlet problem.
%Moreover, we show that comparing temperature {\it differences} instead of absolute values of  temperature has a strong impact on the performance of the algorithm.

The rest of this paper is organized as follows. In section \ref{sec:dir}, we introduce the Dirichlet problem on graphs. Section \ref{sec:algo} describes our algorithm for   node classification. The analysis showing the 
consistency of our algorithm on a simple block model is presented in 
section \ref{sec:model}.  Section \ref{sec:exp} presents some experimental results and section \ref{sec:conc} concludes the paper.

%%%%%%%%%%%%%%%%%%%%%%%%%%%%%%%%%%%%%%%%%%%%%%%%%%%%%%%%%%%%%%%%%%%%%%%%%%%%%%%%%%%%%%%%%%%%%%%%%%%%%%%%%
%%%%%%%%%%%%%%%%%%%%%%%%%%%%%%%%%%%%%%%%%%%%%%%%%%%%%%%%%%%%%%%%%%%%%%%%%%%%%%%%%%%%%%%%%%%%%%%%%%%%%%%%%
\section{Dirichlet problem on graphs}
\label{sec:dir}

In this section, we introduce the Dirichlet problem on graphs % show its interpretation in terms of   random walk in the graph 
and characterize the solution, used later in the analysis. %The difference with the heat kernel is also highlighted.

\subsection{Heat equation}
\label{ssec:pb}

Consider an undirected graph $G$ with $n$ nodes indexed from $1$ to $n$. Denote by $A$ its adjacency matrix. This is a symmetric matrix with non-negative entries.   Let $d = A1$ be the degree vector, which is assumed positive, and $D = \text{diag}(d)$. The Laplacian matrix is defined by:
$$
L = D - A.
$$

Now let  $S$ be some strict subset  of $\{1, \dots, n\}$ and assume that each node $i\in S$ is assigned some fixed temperature $T_i$. %We refer to these nodes as the {\it seeds}. 
We are interested in the evolution of the temperatures of the other nodes, we refer to as the {\it free} nodes. We assume that heat exchanges occur through each edge of the graph proportionally to the temperature difference between the corresponding nodes, so that:
$$
\forall i \notin S, \quad \dfrac{dT_i}{dt} = \underset{j=1}{\overset{n}{\sum}}A_{ij}(T_j - T_i),
$$
that is,
$$
\forall i \notin S,\quad  \dfrac{dT_i}{dt} = -(LT)_i,
$$
where $T$ is the vector of temperatures, of dimension $n$. This is the heat equation in discrete space. At equilibrium, the vector  $T$ satisfies Laplace’s equation:
\begin{equation}
    \label{eq:laplace}
    \forall i \notin S,\quad (LT)_i = 0.
\end{equation}
With the boundary constraint giving the temperature $T_i$ for each node $i \in S$, this defines a Dirichlet problem. Observe that Laplace's equation \eqref{eq:laplace} can be written equivalently:
\begin{equation}
    \label{eq:laplace2}
    \forall i \notin S,\quad T_i = (PT)_i,
\end{equation}
where $P =D^{-1}A$ is the transition matrix of the random walk in the graph.

\subsection{Solution to the Dirichlet problem}
\label{ssec:sol}

We now characterize the solution to the Dirichlet problem  \eqref{eq:laplace}. Without any loss of generality, we assume that free nodes  (i.e., not in $S$) are indexed from 1 to $n-s$ so that the vector of temperatures can be written
$$
T = \begin{bmatrix}X\\ Y\end{bmatrix},
$$
where $X$ is the  vector of temperatures of free nodes at equilibrium, of dimension $n-s$, and $Y$ is the vector of temperatures of the seeds, of dimension $s$. Writing 
the transition matrix in block form as
$$
P = \begin{bmatrix}Q & R \\ \cdot & \cdot\end{bmatrix},
$$
it follows from \eqref{eq:laplace2} that:
\begin{equation}
    \label{eq:linear}
    \quad X = QX + RY,
\end{equation}
so that:
\begin{equation}
    \label{eq:solinear}
    \quad X = (I-Q)^{-1}RY.
\end{equation}
Note that the inverse of $I-Q$ exists whenever  the graph is connected \cite{chung}.
The solution to the Dirichlet problem exists and is unique.

\section{Node classification algorithm}
\label{sec:algo}

In this section, we introduce a node classification algorithm based on  the Dirichlet problem.
The objective is  to infer the labels of all nodes given the  labels of a few nodes called the \textit{seeds}.
Our algorithm is a simple modification of the popular method proposed by \cite{zhu2003semi}. Specifically, we propose to {\it center}  temperatures  before label assignment.

\subsection{Binary classification}

When there are only two different labels, say  0 and 1, the classification follows from the solution of a single Dirichlet problem.
The idea is to set at 0  the temperature of  seeds with label 0  and  at 1  the temperature of  seeds with label 1.
The solution to this Dirichlet problem gives temperatures between 0 and 1 to the free nodes, as illustrated by Figure \ref{fig:karate} for the Karate Club graph   \cite{zachary1977information}.

\begin{figure}[h]
    \centering
 \subfloat[Ground truth]{
    \includegraphics[width=0.47\linewidth]{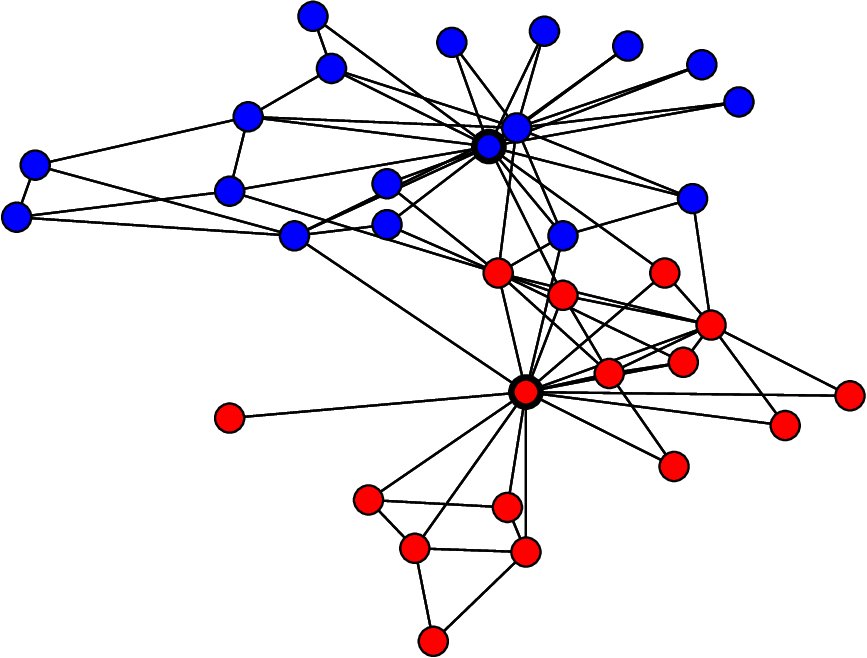}}
    \subfloat[Solution to the Dirichlet problem]{ \includegraphics[width=0.47\linewidth]{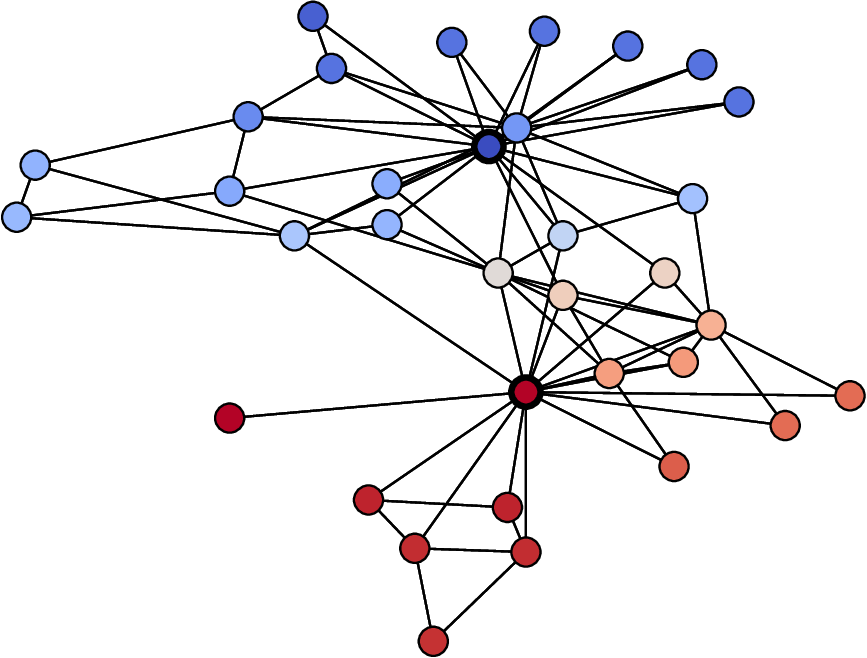}}
    \caption{Binary classification of the Karate Club graph with 2 seeds (indicated with a black circle). Blue nodes have label 0, red nodes have label 1.}
    \label{fig:karate}
\end{figure}

A natural decision rule is to use a threshold of $1/2$ for classification: any free node with temperature above $1/2$ at equilibrium is  assigned label 1, while any other free node is assigned label 0.  The analysis of Section \ref{sec:model} suggests that it is preferable to set the threshold to the mean temperature at equilibrium, 
\begin{equation}\label{eq:thresh}
\bar T = \frac 1 n \sum_{i=1}^n T_i.
\end{equation}
Specifically, any free node with temperature above $\bar T$ at equilibrium is assigned  label 1, while any  other free node is assigned  label 0. Equivalently, temperatures are {\it centered} by their mean before classification: after centering, free nodes with positive temperature are  assigned  label 1, the others are assigned  label 0.

It is worth noting that the threshold \eqref{eq:thresh} is the mean temperature of {\it all} nodes at equilibrium, including seed nodes. Another option, suggested by the {\it class mass normalization} step of \cite{zhu2003semi} for instance, is to set the threshold at  the mean temperature of {\it free} nodes at equilibrium. This variant of the algorithm is not provably consistent, however. 
%Note that the temperature of each node can be used to assess the confidence in the classification: the closer the temperature to  the mean, the lower the confidence.
%This is illustrated by Figure \ref{fig:karate} (the lighter the color, the lower the confidence). In this case, only one node is misclassified and has indeed a temperature close to the mean.

\subsection{Multi-class classification}
In the general case with $K$   labels, we use 
 a \textit{one-against-all} strategy: the seeds of each label alternately serve as hot sources (temperature 1) while all the other seeds serve as cold sources (temperature 0). After centering the temperatures (so that the mean temperature of each diffusion is equal to 0), each node is assigned  the label that maximizes its temperature. This algorithm, we refer to as the Dirichlet classifier, is parameter-free.

\begin{algorithm}[ht]
\caption{Dirichlet classifier}
\begin{algorithmic}[1]
\REQUIRE Seed set $S$ and associated labels $y\in\{1,\ldots,K\}$
\FOR{$k$ in $\{1,\ldots,K\}$}
    \STATE $T = 0$
    \FOR{$i \in S$}
        \IF{$y_i = k$}
            \STATE $T_i = 1$
        \ENDIF
    \ENDFOR
    \STATE $T \leftarrow \text{Dirichlet}(S, T)$
    \STATE $\Delta{(k)} \leftarrow T - \frac 1 n \sum_{i=1}^n T_i$
\ENDFOR
    \FOR{$i \not \in S$}
\STATE $\hat y_i = \arg\max_{k=1,\ldots,K}(\Delta_i{(k)})$
    \ENDFOR
\RETURN $\hat y$, predicted labels of free nodes (outside $S$)
\end{algorithmic}
\label{algo:maxdiff}
\end{algorithm}

The solution to the Dirichlet problem (line 8 of the algorithm) can be obtained  either from \eqref{eq:solinear} or from  iterations of the fixed-point equation \eqref{eq:linear}.

\iffalse
\paragraph{Preprocessing}
A standard preprocessing for semi-supervised learning on graphs is to disconnect all seeds belonging to different classes and connect members of the same class \cite{kamvar2003spectral}. This is, of course, compatible with the proposed algorithm.
\fi

%\subsection{Time complexity}
%
%The time complexity depends on the algorithm used to solve the Dirichlet problem. We here focus on the approximate solution by successive iterations of \eqref{eq:linear}.
%Let $m$ be the number of edges of the graph. Using the  Compressed Sparse Row format for the adjacency matrix, each matrix-vector product  has a complexity of $O(m)$. The complexity of Algorithm \ref{algo:maxdiff} is then $O(NKm)$, where $N$ is the number of iterations.  Note that the $K$ Dirichlet problems are independent and can thus be computed in parallel.

\section{Analysis}
\label{sec:model}

In this section, we prove the consistency of Algorithm \ref{algo:maxdiff} on a simple block model. In particular, we highlight the importance of temperature centering (line 9 of the algorithm)  for the consistency of the algorithm.

\subsection{Block model}

Consider a graph of $n$ nodes consisting of $K$ blocks  of respective sizes $n_1,\ldots,n_K$, forming a partition of the set of nodes.  There are  $s_1,\ldots,s_K$ seeds in these blocks, which have labels $1,\ldots,K$, respectively. Intra-block edges have weight $p$ and inter-block edges have weight $q$. 
We expect  the algorithm to assign label $k$ to all nodes of  block $k$,  for all $k=1,\ldots,K$, whenever $p>q$, i.e.,  the blocks are {\it assortative} \cite{newman2003mixing}.

\subsection{Dirichlet problem}

Consider  the Dirichlet problem when  the temperature of the $s_1$ seeds of block 1 is set to 1 and the temperature  of the other seeds is set to 0. We have an explicit  solution to this Dirichlet problem, given by Lemma \ref{prop:dirichlet}. All proofs are deferred to the appendix.

\begin{lemma}
\label{prop:dirichlet}
Let $T_k$ be the temperature of free nodes of  block $k$ at equilibrium. We have:
    \begin{align*}
(s_1(p-q) + nq) T_1  &= s_1 (p-q) + n\bar T q,\\
(s_k(p-q) + nq) T_k  &= n\bar T q\quad \quad k=2,\ldots,K,
\end{align*}
where $\bar T$ is the average temperature, given by:
$$
 \bar T =\frac 1 n \sum_{i=1}^n T_i =  \left(\frac{s_1}{n} \frac{n_1(p-q) + nq}{ s_1(p-q) + nq}\right) / \left(1-\sum_{k=1}^K \frac{(n_k - s_k)q}{s_k(p-q)+nq}\right).
$$
\end{lemma}

\subsection{Classification}
\label{ssec:clf}

We now state the main result of the paper: the Dirichlet classifier is a consistent algorithm for the block model, in the sense that all nodes are correctly classified whenever $p> q$. 

\begin{theorem}\label{theo:dirichlet}
If $p> q$, then the predicted label  of each free node $i$ of  block $k$ is $\hat y_i=k$,  for any  $n_1,\ldots,n_K$ (label distribution) and $s_1,\ldots,s_K$ (seed distribution).
\end{theorem}

Observe that the temperature centering is critical for consistency. In the absence of centering,  free nodes of block 1 are correctly classified if and only if their temperature is the highest in the Dirichlet problem associated with label 1.
In view of Lemma \ref{prop:dirichlet}, this means that for all $k=2,\ldots,K$,
\begin{align*}
&s_1 q \frac{n_1(p-q) + nq}{s_1(p-q) + nq} + s_1(p-q)\left(1-\sum_{j=1}^K \frac{(n_j - s_j)q}{s_j(p-q)+nq}\right)\\
&> s_k q \frac{n_k(p-q) + nq}{s_k(p-q) + nq}.
\end{align*}
This condition might be violated even if $p>q$, depending on the parameters  $n_1,\ldots,n_K$ and $s_1,\ldots,s_K$. 
In the simplest case of $K=2$ blocks, with $p=10^{-1}$ and $q = 10^{-2}$ for instance, the classification is incorrect in the following two asymmetric cases:
\begin{description}
\item{\bf Seed asymmetry}  (blocks of same size but different number of seeds): \\ $n_1 = n_2=100; s_1 = 10, s_2=5$,
\item{\bf Label asymmetry} (blocks with the same number of seeds but different sizes):\\ $n_1 = 100, n_2=10; s_1 = s_2=5$ .
\end{description}
%In the practically interesting case where $s_1 << n_1,\ldots,s_K << n_K$ for instance (low fractions of seeds), the condition requires:
%$$
%s_1  (n_1(p-q) + nq) > s_k  (n_k(p-q) + nq).
%$$
%For blocks of same size, $n_1=\ldots=n_K$, this means that only blocks with the largest number of seeds are correctly classified.  
%The classifier is biased towards labels with a large number of seeds. 
This sensitivity of the algorithm to both forms of asymmetry will be confirmed by the experiments. %It explains the bad performance of  
The step of temperature centering is crucial for consistency.

\section{Experiments}
\label{sec:exp}

In this section, we show the impact of temperature centering on the quality of classification using both synthetic and real data.
The Python code is available as a Jupyter notebook in Python\footnote{\url{https://perso.telecom-paris.fr/bonald/notebooks/diffusion.ipynb}}, making the experiments fully reproducible.
%We do {\it not} provide a general benchmark of classification methods as the focus of the paper is on the impact of temperature centering in heat diffusion methods. 
%We focus on three algorithms: the vanilla algorithm (without temperature centering), the weighted version proposed by \cite{zhu2003semi} (also without temperature centering, but with weights  depending on the distribution of labels within seeds) and our algorithm (with temperature centering).https://perso.telecom-paris.fr/bonald/notebooks/diffusion.ipynb

%All datasets and codes are  available in the supplementary material, making the experiments fully reproducible.

%The parameter of the heat kernel is set to $N=10$ after grid search in $N \in \{1, 3, 5, 10, 25\}$ on the Cora dataset. We also use $N=10$ for the Dirichlet-based methods in order to have a similar time complexity.

\subsection{Synthetic data}

We first use the stochastic block model \cite{airoldi2008mixed} to generate graphs with an underlying structure in clusters. This is the stochastic version of the block model used in the analysis. There are $K$ blocks of respective sizes $n_1,\ldots,n_K$. Nodes of the same block are connected with probability  $p$ while nodes in different blocks are connected probability $q$. Nodes in block $k$ have label $k$.
We denote by $s_k$ the number of  seeds in block $k$ and by $s$  the total number of seeds.

%Another variant proposed by \cite{zhu2003semi} consists in rescaling   the temperature vector  by the  weight of the considered label in the seeds (see equation (9) in this paper).

\begin{figure}[h]
    \centering
    \subfloat[Seed asymmetry]{\includegraphics[width=0.48\linewidth]{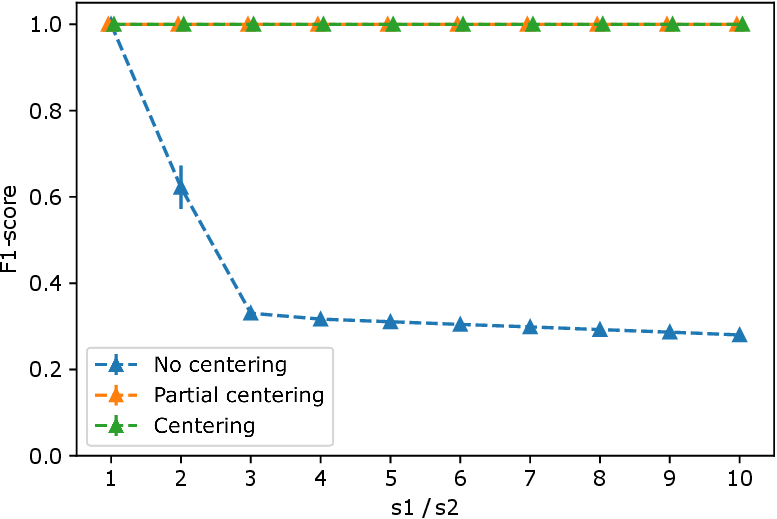}}
    \subfloat[Label asymmetry]{\includegraphics[width=0.48\linewidth]{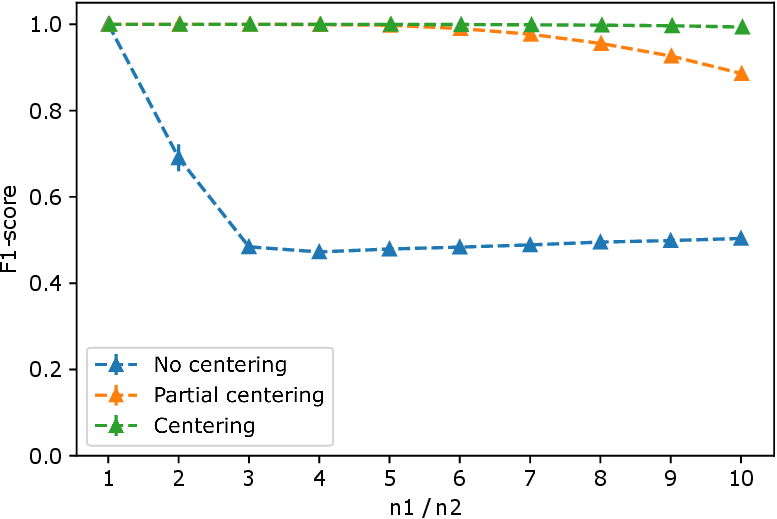}}
    \caption{F1 scores on the stochastic block model (2 labels).}
    \label{fig:sbm}
\end{figure}

We first compare the performance of the algorithms on a binary classification task ($K=2$) for a graph of $n=10\,000$ nodes with $p=10^{-2}$ and $q=10^{-3}$, in two different settings:
\begin{description}
\item{\bf Seed asymmetry} (blocks of same size but different number of seeds): \\
   $n_1= n_2 = 5000$,    $s_2 = 250$,   ratio $s_1/s_2 \in \{1, 2, \dots,10\}$.\\
    (5\% of nodes in block 2 are seeds)
\item{\bf Label asymmetry}  (blocks with the same number of seeds but different sizes):\\  number of nodes  $n=10\,000$, ratio  $n_1/n_2 \in \{1, 2, \dots, 10\}$, $s_1 = s_2= 250$.\\ (5\% of all nodes are seeds)
\end{description}

For each configuration, the experiment is repeated 100 times. Randomness comes both from the generation of the graph and from the selection of the seeds. We report the F1-scores in Figure \ref{fig:sbm} (mean $\pm$ standard deviation). Observe that the variability of the results is very low due to the relatively large size of the graph. As expected, the centered version is much more robust to both forms of asymmetry.  The  variant called {\it partial centering}, where the mean temperature is computed over free nodes only, tends to be less robust to label asymmetry.

We show in Figure \ref{fig:sbm/multi} the same results for $K = 5$ blocks, still with  $n=10\,000$ nodes,  $p=10^{-2}$ and $q=10^{-3}$.  Blocks $2, 3, 4, 5$ have the same size and the same number of seeds. For the experiments on seed asymmetry, each block has $2\,000$ nodes and $5\%$ of nodes in blocks  $2, 3, 4, 5$ are seeds; we only vary the number of seeds in block 1. For the experiments on label asymmetry, there is the same number of seeds for each label, corresponding to an average proportion of  $5\%$ of all nodes.  The performance metric is the F1-score averaged over the 5 labels. The conclusions are the same as with 2 labels.%: the algorithm is much more robust to asymmetry with temperature centering.

\begin{figure}[h]
    \centering
   \subfloat[Seed asymmetry]{\includegraphics[width=0.48\linewidth]{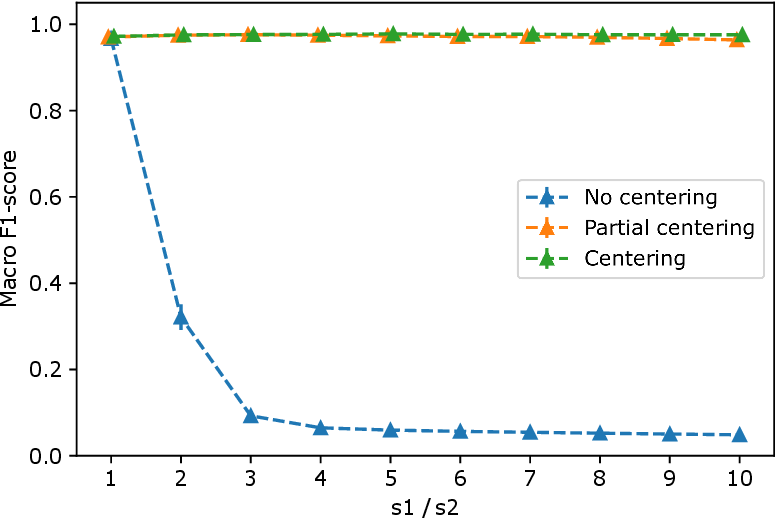}}
    \subfloat[Label asymmetry]{\includegraphics[width=0.48\linewidth]{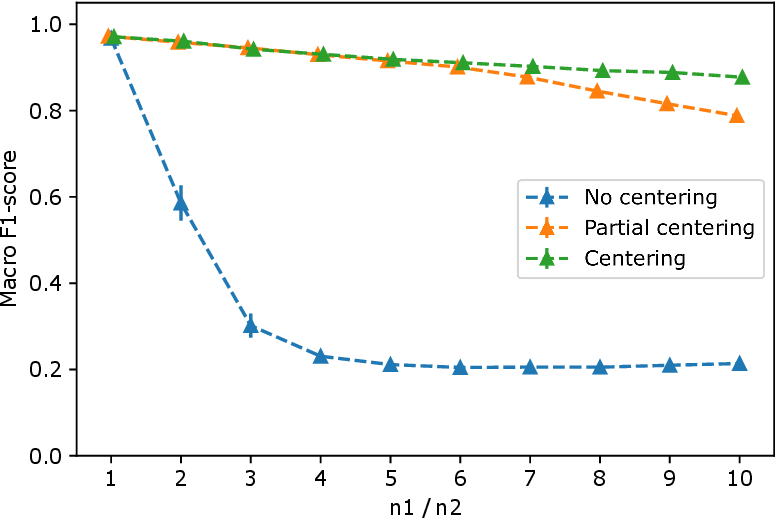}}
    \caption{Macro F1-scores on the  stochastic block model (5 labels).}
    \label{fig:sbm/multi}
\end{figure} 

%\clearpage

\subsection{Real data}
\label{ssec:datasets}

We  now focus on  real datasets available from the SNAP collection\footnote{\url{https://snap.stanford.edu/}}  and the  NetSet\footnote{\url{https://netset.telecom-paris.fr/}}  collection, restricting to graphs having ground-truth labels. All graphs are considered undirected.

\begin{table}[ht]
    \centering
    \caption{Overview of the datasets.}
    \begin{tabular}{c|ccc}
        \toprule
        Dataset & \#nodes & \#edges & \#classes  \\
\midrule
        Cora & $2\,708$ & $5\,278$ & 7  \\ 
        Citeseer & $3\,264$ & $4\,536$ & 6 \\
        PubMed & 19\,717& 44\,325& 3\\
        Email& 1\,005 &16\,385  &42\\
        PolBlogs & 1\,490& 16\,716& 2\\
        WikiSchools & $4\ 403$ & $100\,329$ & 16 \\
        WikiVitals & $10\,011$ & $654\,502$ & 11 \\
        WikiVitals+ & $45\,179 $ & $3\,079\,335$ & 11 \\
     %   DBLP & $317\,080$ & $1\,049\,866$ & 5000 & 29\\
       % Amazon & $334\,863$ & $925\,872$ & 5000 & 5
       \bottomrule
    \end{tabular}
    \label{tab:datasets}
\end{table}

For each dataset, we select seeds uniformly at random.
The process is repeated 100 times. 
The macro-F1 scores (i.e., F1-scores averaged over all classes) are shown in 
Table~\ref{tab:macrof1/node}  for seeds representing  $5\%, 10\%$ or $20\%$ of the nodes. We see that the centered version  outperforms the standard version over all datasets. The performance gains are substantial  for the largest graphs, extracted from Wikipedia. The variance is also lower in all cases, showing the robustness of the algorithm.   Additional results, not reported here, tend to  show that
the variant selected for temperature centering (based on either all nodes or free nodes) has a marginal impact on performance. 

%Weighted version not shown since very similar.

%, especially on the Wikipedia datasets where the classes are not as well separated then in the social datasets. For example, the class \textit{person} in the Wikipedia dataset is rather intertwined with the other classes.
%On the other hand, the rescaling proposed by Zhu seems to have little effect on the performance with respect to the vanilla method.

%Besides, note that the edge sampling seeding seems to lead to better classification with respect to node sampling on the social datasets but not on the Wikipedia ones. A possible explanation is that the classes are less balanced among the high degree nodes for the Wikipedia datasets than they are for the social ones.

\begin{table}[ht]
    \centering
    \caption{Macro-F1 scores (mean $\pm$ standard deviation) without and with temperature centering.}    
%  \subfloat[1\% of seeds] { \begin{tabular}{l|cc|c}
%    \toprule
%        Dataset & Vanilla & Centered & Variation\\
%        \midrule
% Cora & {\bf 0.45} $\pm$ 0.06 & 0.44  $\pm$ 0.04 & $-4\%$ \\
%CiteSeer & {\bf 0.30} $\pm$ 0.04 & 0.27 $\pm$ 0.04 & $-11\%$ \\
%PubMed & {\bf 0.72}  $\pm$ 0.02 & 0.67  $\pm$ 0.03 & $-6\%$ \\
%Email & 0.05 $\pm$ 0.02 & {\bf 0.07} $\pm$ 0.02 & $+33\%$ \\
%PolBlogs & 0.48 $\pm$ 0.12 & {\bf 0.57} $\pm$ 0.01 & $+18\%$ \\
%WikiSchools & 0.15 $\pm$  0.06 & {\bf 0.28} $\pm$  0.04 & $+82\%$ \\
%WikiVitals & 0.38 $\pm$  0.06 & {\bf 0.51} $\pm$  0.03 & $+35\%$ \\
%WikiVitals+ & 0.44 $\pm$  0.04 & {\bf 0.51} $\pm$  0.01 & $+16\%$ \\
%    \bottomrule
%        \end{tabular}}
%       

  \subfloat[5\% of seeds] { \begin{tabular}{l|cc|c}
    \toprule
        Dataset & No centering & Centering & Variation\\
        \midrule
 Cora & {0.69} $\pm$ 0.02 & {\bf 0.71} $\pm$ 0.02 & $+2\%$ \\
Citeseer & {\bf 0.48} $\pm$ 0.01 & {\bf 0.48} $\pm$ 0.01 & $0\%$ \\
PubMed & {0.76} $\pm$ 0.01 & {\bf 0.78} $\pm$ 0.01 & $+2\%$ \\
Email & 0.12 $\pm$ 0.04 & {\bf 0.22} $\pm$ 0.03 & $+85\%$ \\
PolBlogs & 0.82 $\pm$ 0.12 & {\bf 0.87} $\pm$ 0.01 & $+7\%$ \\
WikiSchools & 0.08 $\pm$ 0.06 &  {\bf 0.44} $\pm$ 0.03 & $+472\%$ \\
WikiVitals & 0.29 $\pm$ 0.06 &  {\bf 0.63} $\pm$ 0.02 & $+116\%$ \\
WikiVitals+ & 0.31 $\pm$ 0.03 & {\bf 0.65} $\pm$ 0.01 & $+112\%$ \\
    \bottomrule
        \end{tabular}}
        
                \subfloat[10\% of seeds] { \begin{tabular}{l|cc|c}
    \toprule
        Dataset & No centering & Centering & Variation\\
        \midrule
 Cora & {0.74} $\pm$ 0.02 & {\bf 0.75} $\pm$ 0.01 & $+1\%$ \\
Citeseer & {\bf 0.52} $\pm$ 0.01 & {\bf 0.52} $\pm$ 0.01 & $0\%$ \\
PubMed & {0.78} $\pm$ 0.01 & {\bf 0.79} $\pm$ 0.00 & $+1\%$ \\
Email & 0.21 $\pm$ 0.04 & {\bf 0.31} $\pm$ 0.03 & $+43\%$ \\
PolBlogs & {0.86} $\pm$ 0.02 & {\bf 0.87} $\pm$ 0.01 & $+1\%$ \\
WikiSchools & 0.13 $\pm$ 0.04 & {\bf 0.50} $\pm$ 0.02 & $+295\%$ \\
WikiVitals & 0.43 $\pm$ 0.04 & {\bf 0.67} $\pm$ 0.01 & $+57\%$ \\
WikiVitals+ & 0.61 $\pm$ 0.01 & {\bf 0.68} $\pm$ 0.01 & $+12\%$ \\
    \bottomrule
        \end{tabular}}

 \subfloat[20\% of seeds] { \begin{tabular}{l|cc|c}
    \toprule
        Dataset & No centering & Centering & Variation\\
        \midrule
Cora & {\bf 0.78} $\pm$ 0.01 & {\bf 0.78} $\pm$ 0.01 & $0\%$ \\
Citeseer & {\bf 0.57} $\pm$ 0.01 & {\bf 0.57} $\pm$ 0.01 & $0\%$ \\
PubMed & {\bf 0.80} $\pm$ 0.00 & {\bf 0.80} $\pm$ 0.00 & $0\%$ \\
Email & 0.32 $\pm$ 0.03 & {\bf 0.40} $\pm$ 0.02 & $+24\%$ \\
PolBlogs &  {\bf 0.87} $\pm$ 0.01 &  {\bf 0.87} $\pm$ 0.01 & $0\%$ \\
WikiSchools & 0.27 $\pm$ 0.03 & {\bf 0.57} $\pm$ 0.02 & $+110\%$ \\
WikiVitals & 0.58 $\pm$ 0.02 & {\bf 0.70} $\pm$ 0.01 & $+22\%$ \\
WikiVitals+ & 0.65 $\pm$ 0.01 & {\bf 0.71} $\pm$ 0.00 & $+9\%$ \\
    \bottomrule
        \end{tabular}}

%     \subfloat[20\% of seeds] { \begin{tabular}{l|cc|c}
%    \toprule
%        Dataset & Vanilla & Centered & Variation\\
%        \midrule
%cora & 0.73 & 0.01 & 0.73 & 0.01 & -0.0 \\
%citeseer & 0.57 & 0.01 & 0.57 & 0.01 & -0.0 \\
%pubmed & 0.81 & 0.0 & 0.81 & 0.0 & 0.0 \\
%email & 0.44 & 0.03 & 0.48 & 0.03 & 0.09 \\
%polblogs & 0.57 & 0.01 & 0.57 & 0.0 & 0.0 \\
%wikischools & 0.49 & 0.02 & 0.64 & 0.01 & 0.31 \\
%wikivitals & 0.71 & 0.01 & 0.74 & 0.01 & 0.04 \\
%wikivitals+ & 0.67 & 0.0 & 0.67 & 0.0 & -0.01 \\
%    \bottomrule
%        \end{tabular}}

    \label{tab:macrof1/node}
\end{table}

\section{Conclusion}
\label{sec:conc}

We have proposed a novel approach to node classification based on heat diffusion. Specifically, our technique consists   in  centering the temperatures of each solution to the  Dirichlet problem before classification. We have proved the consistency of this algorithm on a simple block model and  shown  that the temperature centering brings significant performance gains on real datasets. This is  a crucial observation  given the popularity of the algorithm.
%The algorithm outperforms heat kernel methods while being  parameter-free. 

The question of the consistency of the algorithm when the mean temperature is computed over free nodes  (instead of all nodes) remains open. Another interesting research perspective is to extend our proof of consistency of the algorithm to {\it stochastic} block models, where edges are drawn at random  \cite{airoldi2008mixed} .

%\clearpage

%% The file named.bst is a bibliography style file for BibTeX 0.99c

\clearpage

\appendix

\section*{Appendix}

\section{Proof of Lemma \ref{prop:dirichlet}}

\begin{proof}
In view of (2), we have:
\begin{align*}
(n_1(p-q) + nq) T_1 &= s_1 p + (n_1-s_1)pT_1 + \sum_{j\ne 1} (n_j - s_j) qT_j,\\
(n_k(p-q) + nq) T_k& = s_1 q + (n_k-s_k)pT_k + \sum_{j\ne k} (n_j - s_j) qT_j,
\end{align*}
for $ k=2,\ldots,K$.
We deduce:
\begin{align*}
(s_1(p-q) + nq) T_1  &= s_1 p + Uq,\\
(s_k(p-q) + nq) T_k  &= s_1 q + Uq\quad \quad \forall k=2,\ldots,K,
\end{align*}
with 
$$
U = \sum_{j=1}^K  (n_j - s_j) T_j.
$$
The proof then follows from the fact that 
$$
n \bar T =  s_1 + \sum_{j=1}^K  (n_j - s_j) T_j = s_1 + U.
$$
\end{proof}

\section{Proof of Theorem \ref{theo:dirichlet}}

\begin{proof}
Let $\Delta^{(1)}_k = T_k - \bar T$ be the deviation of temperature of non-seed nodes of block $k$ for the Dirichlet problem associated with label 1.
    In view of Lemma \ref{prop:dirichlet}, we have:
\begin{align*}
(s_1(p-q) + nq) \Delta^{(1)}_1  &= s_1 (p-q) (1-\bar T),\\
(s_k(p-q) + nq)\Delta^{(1)}_k  &= -s_k(p-q) \bar T \quad \quad k=2,\ldots,K,
\end{align*}
For $p>q$, using the fact that $\bar T \in (0,1)$, we get $\Delta^{(1)}_1 > 0$ and $\Delta^{(1)}_k<0$ for all $k=2,\ldots,K$. By symmetry, for each label $l = 1,\ldots,K$,
$\Delta^{(l)}_l > 0$ and $\Delta^{(l)}_k<0$ for all $k\ne l$.
We deduce that for each block $k$, $\hat y_i=\arg\max_{l}\Delta^{(l)}_k = k$ for each free node $i$ of  block $k$.
\end{proof}

\clearpage

\bibliographystyle{spmpsci}
\bibliography{biblio}

\end{document}